%% file: main.tex
\documentclass[sigconf, anonymous=false,review=false]{acmart}
\renewcommand\footnotetextcopyrightpermission[1]{} 
\AtBeginDocument{%
  \providecommand\BibTeX{{%
    \normalfont B\kern-0.5em{\scshape i\kern-0.25em b}\kern-0.8em\TeX}}}

\copyrightyear{2023}
\acmYear{2023}
\setcopyright{rightsretained}
\acmConference[GECCO '23]{Genetic and Evolutionary Computation Conference}{July 15--19, 2023}{Lisbon, Portugal}
\acmBooktitle{Genetic and Evolutionary Computation Conference (GECCO '23), July 15--19, 2023, Lisbon, Portugal}\acmDOI{10.1145/3583131.3590453}
\acmISBN{979-8-4007-0119-1/23/07}

\input{_used_packages.tex}

\input{_macros.tex}

\allowdisplaybreaks

\begin{document}

\title{Analysis of the (1+1) EA on LeadingOnes with Constraints}

\author{Tobias Friedrich}
\email{friedrich@hpi.de}
\affiliation{%
  \institution{Hasso Plattner Institute\\ University of Potsdam}
  \city{Potsdam}
  \country{Germany}
}

\author{Timo Kötzing}
\email{timo.koetzing@hpi.de}
\affiliation{%
  \institution{Hasso Plattner Institute\\ University of Potsdam}
  \city{Potsdam}
  \country{Germany}
}

\author{Aneta Neumann}
\email{aneta.neumann@adelaide.edu.au}
\affiliation{%
  \institution{Optimisation and Logistics\\ The University of Adelaide}
  \city{Adelaide}
  \country{Australia}
}

\author{Frank Neumann}
\email{frank.neumann@adelaide.edu.au}
\affiliation{%
  \institution{Optimisation and Logistics\\ The University of Adelaide}
  \city{Adelaide}
  \country{Australia}
}

\author{Aishwarya Radhakrishnan}
\email{aishwarya.radhakrishnan@hpi.de}
\affiliation{%
  \institution{Hasso Plattner Institute\\ University of Potsdam}
  \city{Potsdam}
  \country{Germany}
}

\renewcommand{\shortauthors}{ Friedrich and Kötzing, et al.}
\begin{abstract}
Understanding how evolutionary algorithms perform on constrained problems has gained increasing attention in recent years. In this paper, we study how evolutionary algorithms optimize constrained versions of the classical LeadingOnes problem. We first provide a run time analysis for the classical (1+1)~EA on the LeadingOnes problem with a deterministic cardinality constraint, giving $\Theta(n (n-B)\log(B) + n^2)$ as the tight bound. Our results show that the behaviour of the algorithm is highly dependent on the constraint bound of the uniform constraint. Afterwards, we consider the problem in the context of stochastic constraints and provide insights using experimental studies on how the ($\mu$+1)~EA is able to deal with these constraints in a sampling-based setting.
\end{abstract}

\ignore{

\ccsdesc[500]{Computer systems organization~Embedded systems}
\ccsdesc[300]{Computer systems organization~Redundancy}
\ccsdesc{Computer systems organization~Robotics}
\ccsdesc[100]{Networks~Network reliability}
}
\keywords{Evolutionary algorithms, chance constraint optimization, run time analysis, theory.}


\maketitle

\section{Introduction}
Evolutionary algorithms~\cite{DBLP:series/ncs/EibenS15} have been used to tackle a wide range of combinatorial and complex engineering problems. Understanding evolutionary algorithms from a theoretical perspective is crucial to explain their success and give guidelines for their application.
 The area of run time analysis has been a major contributor to the theoretical understanding of evolutionary algorithms over the last 25 years~\cite{DBLP:series/ncs/Jansen13,DBLP:books/daglib/0025643,DBLP:series/ncs/2020DN}. 
 Classical benchmark problems such as OneMax and LeadingOnes have been analyzed 
 in a very detailed way, showing deep insights into the working behaviour of evolutionary algorithms for these problems. In real-world settings, problems that are optimized usually come with a set of constraints which often limits the resources available. Studying classical benchmark problems even with an additional simple constraint such as a uniform constraint, which limits the number of elements that can be chosen in a given benchmark function, poses significant new technical challenges for providing run time bounds of even simple evolutionary algorithms such as the (1+1)~EA.
 

 OneMax and the broader class of linear functions~\cite{DBLP:journals/tcs/DrosteJW02} have played a key role in developing the area of run time analysis during the last 25 years, and run time bounds for linear functions with a uniform constraint have been obtained~\cite{DBLP:journals/tcs/FriedrichKLNS20,DBLP:journals/algorithmica/NeumannPW21}. It has been shown in \cite{DBLP:journals/tcs/FriedrichKLNS20} that the (1+1)~EA needs exponential time optimize OneMax under a specific linear constraint which points to the additional difficulty which such constraints impose on the search process.
 Tackling constraints by taking them as additional objectives has been shown to be quite successful for a wide range of problems. For example, the behaviour of evolutionary multi-objective algorithms has been analyzed for submodular optimization problems with various types of constraints~\cite{DBLP:conf/nips/QianYZ15,DBLP:conf/ijcai/QianSYT17}. Furthermore, the performance of evolutionary algorithms for problems with dynamic constraints has been investigated in \cite{DBLP:journals/tcs/RoostapourNN22,DBLP:journals/ai/RoostapourNNF22}.

Another important area involving constraints is chance constrained optimization, which deals with stochastic components in the constraints. Here, the presence of stochastic components in the constraints makes it challenging to guarantee that the constraints are not violated at all. Chance-constrained optimization problems~\cite{charnes1959chance,miller1965chance} are an important class of the stochastic optimization problems~\cite{beyer2007robust} that optimize a given problem under the condition that a constraint is only violated with a small probability. 
Such problems occur in a wide range of areas, including finance, logistics and engineering~\cite{li2008chance,zhang2011chance,nair2011fleet,hanasusanto2015distributionally}. 
Recent studies of evolutionary algorithms for chance-constrained problems focused on a classic knapsack problem where the uncertainty lies in the probabilistic constraints  \cite{DBLP:conf/gecco/XieHAN019,DBLP:conf/gecco/XieN020}. Here, the aim is to maximise the deterministic profit subject to a constraint which involves stochastic weights and where the knapsack capacity bound can only be violated with a small probability of at most $\alpha$. 
 A different stochastic version of the knapsack problem has been studied in~\cite{DBLP:conf/ppsn/NeumannXN22}. Here profits involve uncertainties and weights are deterministic. In that work, Chebyshev and Hoeffding-based fitness
functions have been introduced and evaluated. These fitness functions discount expected profit values based on uncertainties of the given solutions.

Theoretical investigations for problems with chance constraints have gained recent attention in the area of run time analysis. This includes studies for montone submodular problems~\cite{DBLP:conf/ppsn/NeumannN20} and special instances of makespan scheduling~\cite{DBLP:conf/ppsn/ShiYN22}.
Furthermore, detailed run time analyses have been carried out for specific classes of instances for the chance constrained knapsack problem~\cite{DBLP:conf/foga/0001S19,DBLP:conf/gecco/XieN0S21}. 
\subsection{Our contribution}
In this paper, we investigate the behaviour of the (1+1)~EA for the classical \LO problem with additional constraints. We first study the behaviour for the case of a uniform constraint which limits the number of $1$-bits that can be contained in any feasible solution. Let $B$ be the upper bound on the number of $1$-bits that any feasible solution can have. Then the optimal solutions consists of exactly $B$ leading $1$s and afterwards only $0$s. The search for the (1+1)~EA is complicated by the fact that when the current solution consists of $k<B$ leading $1$s, additional $1$-bits not contributing to the fitness score at positions $k+2, \ldots, n$ might make solutions infeasible. We provide a detailed analysis of such scenarios in dependence of the given bound $B$.

Specifically, we show a tight bound of $\Theta(n^2 + n (n-B) \log(B))$ (see Corollary~\ref{cor:RuntimeOnConstrainedLO}). 
Note that \cite{DBLP:journals/tcs/FriedrichKLNS20} shows the weaker bound of $O(n^2 \log(B))$, which, crucially, does not give insight into the actual optimization process at the constraint. Our analysis shows in some detail how the search progresses. In the following discussion, for the current search point of the algorithm, we call the part of the leading $1$s the \emph{head} of the bit string, the first $0$ the \emph{critical bit} and the remaining bits the \emph{tail}. While the size of the head is less than $B-(n-B)$, optimization proceeds much like for unconstrained LeadingOnes; this is because the bits in the tail of size about $2(n-B)$ are (almost) uniformly distributed, contributing roughly a number of $n-B$ many $1$s additionally to the $B-(n-B)$ many $1$s in the head. This stays in sum (mostly) below the cardinality bound $B$, occasional violations changing the uniform distribution of the tail to one where bits in the tail are $1$ with probability a little less than $1/2$ (see Lemma~\ref{lem:distributionInTail}).

Once the threshold of $B-(n-B)$ many $1$s in the head is passed, the algorithm frequently runs into the constraint. For a phase of equal LeadingOnes value, we consider the random walk of the number of $1$s of the bit string of the algorithm. This walk has a bias towards the bound $B$ (its maximal value), where the bias is light for LeadingOnes-values just a bit above $B-(n-B)$ and getting stronger as this value approaches $B$. Since progress is easy when not at the bound of $B$ many $1$s in the bit string (by flipping the critical bit and no other) and difficult otherwise (additionally to flipping the critical bit, a $1$ in the tail needs to flip), the exact proportion of time that the walk spends in states of less than $B$ versus exactly $B$ many $1$s is very important. In the final proofs, we estimate these factors and have corresponding potential functions reflecting gains (1) from changing into states of less than $B$ many $1$s and (2) gaining a leading~$1$. Bounding these gains appropriately lets us find asymptotically matching upper and lower bounds using the additive drift theorem~\cite{ad}.

In passing we note that two different modifications of the setting yield a better time of $O(n^2)$. First, this time is sufficient to achieve a LeadingOnes-values of $B-c(n-B)$ for any $c>0$ (see Corollary~\ref{cor:cor1}). Second, considering the number of $1$s as a secondary objective (to be minimized) gives an optimization time of $O(n^2)$ (see Theorem~\ref{thm:improved_algo}).

Afterwards, we turn to stochastic constraints and investigate an experimental setting that is motivated by recent studies in the area of chance constraints. We consider LeadingOnes with a stochastic knapsack chance constraint, where the weights of a linear constraint are chosen from a given distribution. In the first setting, the weight of each item is chosen independently according to a Normal distribution $N(\mu, \sigma^2)$. A random sample of weights is feasible if the sum of the chosen sampled weights does not exceed a given knapsack bound~$B$. In any iteration, all weights are resampled independently for all evaluated individuals. Our goal is to understand the maximal stable LeadingOnes value that the algorithm obtains. In the second setting which we study empirically, the weights are deterministically set to~$1$ and the bound is chosen uniformly at random within an interval $[B-\epsilon, B+\epsilon]$, where $\epsilon>0$ specifies the uncertainty around the constraint bound. For both settings, we examine the performance of the $(1+1)$~EA and $(10+1)$-EA for different values of $B$ and show that a larger parent population has a highly positive effect for these stochastic settings.

The paper is structured as follows. In Section~\ref{sec:prelim}, we introduce the problems and algorithms that we study in this paper. We present our run time analysis for the LeadingOnes problem with a deterministic uniform constraint in Section~\ref{sec3}. In section \ref{sec:sec4}, we discuss a way to obtain $\Theta(n^2)$ bound on the run time for the same problem and report on our empirically investigations for the stochastic settings in Section~\ref{sec:experiments}. Finally, we finish with some concluding remarks. Note that some proofs are ommitted due to space constraints.


\section{Preliminaries}
\label{sec:prelim}
In this section we define the objective function, constraints and the algorithms used in our analysis. With $|x|_1$ we denote the number of $1$s in a bit string $x \in \bitstring$.

\subsection{Cardinality Constraint}\label{cac}
Let $f: \bitstring \rightarrow \realnum$, $B \leq n $ and for $x \in \bitstring$, let $x_i$ denote the $i$ -th bit of $x$. In this paper, optimizing $f$ with cardinality constraint $B$ means finding,
$\max_{x \in \bitstring} f(x) \text{ s.t } \sum_{i = 1}^{n} x_i \leq B.$

\subsection{Stochastic Constraint}\label{sc}
Let $f: \bitstring \rightarrow \realnum$, $B \leq n $ and for $x \in \bitstring$, let $x_i$ denote the $i$ -th bit of $x$. In this paper we empirically analyse the following normal stochastic constraint with uncertainty in the weights optimization problem,
$$\max_{x \in \bitstring} f(x) \text{ s.t } \sum_{i = 1}^{n}w_i \cdot x_i \leq B \text{, where } w_i \sim N(\mu, \sigma^2).$$

Let $f: \bitstring \rightarrow \realnum$, $B \leq n $ and for $x \in \bitstring$, let $x_i$ denote the $i$ -th bit of $x$. In this paper we also empirically analyse the following uniform stochastic constraint  with uncertainty in the bound optimization problem, 
$$\max_{x \in \bitstring} f(x) \text{ s.t } |x|_1 \leq y \text{, where } y \sim U(B - \epsilon, B + \epsilon).$$

\subsection{Objective Function}
We consider the \LO function as our objective with cardinality and stochastic constraints for our analysis.

$\LO: \bitstring  \rightarrow \realnum$, is a function which maps a bit string of length $n$ to number of $1$s before the first $0$ in the bit string. For every $x \in \bitstring$,
$\LO(x) = \sum_{i = 1}^{n} \prod_{j = 1}^{i} x_j.$

\subsection{(\texorpdfstring{$\mu$+1)}{Lg} EA}
The \muoea on a real valued fitness function $f$ with constraint $B$ is given in Algorithm \ref{alg:mu+1EA}. The \muoea at each iteration maintains a population of size $\mu$. The initial population $P_0$ has $\mu$ random bit strings chosen uniformly. At each iteration $t > 0$, a bit string is chosen uniformly at random from $P_t$ followed by a mutation operation which flips each bit of the chosen bit string with probability $\frac{1}{n}$. The mutated bit string is added to $P_t$ and the bit string with the least fitness among the $\mu + 1$ individuals is removed. Since we can also sample a bit string which violates the constraint, we consider the following function for optimization.

 \[g(x) = \begin{cases} 
      f(x), & \text{if $|x|_1 < B$};\\
     B - |x|, & \text{otherwise.}
   \end{cases}
\]

\begin{myAlgorithm}
	$P_0 \gets $ $\mu$ individuals from $\bitstring $ chosen u.a.r.\;\label{step:initialization_EA}
    $t = 0$\;
	\While{stopping criterion not met}{
    $x \gets $ uniform random bit string from $P_t$\;
	$y$ $\gets$ flip each bit of $x$ independently with probab.\ $1/n$\;
    $P_t = P_t \cup \{y\}$\;
    $P_{t+1} = P_t \setminus \{$an individual $x \in P_t$ with least $g(x)$ value$\}$\;\label{step:apdopt_current_best_EA}
    $t = t + 1$\;
	}
	\caption{\muoea on fitness function $f$ with constraint B}\label{alg:mu+1EA}
\end{myAlgorithm}

\section{Unmodified Setting}
\label{sec3}
In this section we give a tight analysis of the \ooea on the objective \LO with cardinality constraint $B$. 

We start with a technical lemma which we need for our proof of the upper bound.
\begin{lem}\label{lem:UBFB}
For $t \geq 0$, let $x_t$ denote the parent bit string at $t$-th iteration while \ooea is optimizing \LO with the cardinality constraint B. And for $t >0$, let $A_{t}$ denote the event that $|x_{t+1}|_1 = B$ and $LO(x_{t+1}) = LO(x_t)$. Then $Pr(A_{t} \bigm\vert |x_t|_1 < B) \leq \frac{n - B}{n}$.
\end{lem}
\begin{proof}
First note that, if $|x_t|_1 = k < B$ and $C_{t}$ denote the event that $x_{t+1}$ is formed by flipping $B - k$ number of $0$ bits to $1$ out of $n - k - 1$ (except the left most $0$) number of $0$ bits, then 
\[Pr(A_{t} \bigm\vert |x_t|_1 < B ) \leq Pr(C_{t} \bigm\vert |x_t|_1 < B).\]
The event $A_{t}$ is a sub-event of $C_{t}$, since in the event $C_{t}$ we do not have any restriction on the bits other than  $B - k$ number of $0$ bits out of $n - k - 1$ number of them and we have to flip at least $B - k$ number of $0$ bits to $1$ to get the desired $x_{t+1}$ in the event $A_{t}$. Hence,
\begin{align*}
Pr(C_{t}) 
    &= \binom{n - k - 1}{B - k} \left(\frac{1}{n}\right)^{B - k}\\
    &= \frac{(n - k - 1) \cdot (n - k - 2) \cdots (n - B)}{1 \cdot 2 \cdots (B - k)} \cdot \left(\frac{1}{n}\right)^{B - k} \leq \frac{n - B}{n}.
\end{align*}
The last inequality holds because, for every $r > 0$, $\frac{n - k - r}{n} \leq 1$.
\end{proof}

In the Theorem \ref{thm:ub} below we give an upper bound on the expected run time of the \ooea on \LO with cardinality constraint $B$. Later we show that this bound is tight by proving a matching lower bound.
\begin{thm}\label{thm:ub}
Let $n, B \in \natnum$ and $B < n$. Then the expected optimization time of the \ooea on \LO with cardinality constraint $B$ is $O\left(n^2 + n(n-B)\log B\right).$ 
\end{thm}
\begin{proof}
From \cite[Lemma~3]{lo}, we know that the \ooea is expected to find a feasible solution within $O(n\log(n/B))$ iterations. Now we calculate how long it takes in expected value to find the optimum after a feasible solution is sampled.

To do this, we construct a potential function that yields an drift value greater than $1$ at each time $t$ until the optimum is found. For $i \in \{0, \cdots , B\}$, let $g_{B}(i)$ be the potential of a bit string $x \in \bitstring$ with exactly $B$ number of $1$s and $LO(x) = \LO(x) = i$. For $i \in \{0, \cdots , B - 1\}$, let $g_{<B}(i)$ be the potential of a bit string  $x \in \bitstring$ with less than $B$ number of $1$s and $LO(x) = i$.

Let 
$g_B(0) = 0 \text{ and } g_{<B}(0) = \frac{en}{B}.$
And for every $i \in \{1, \cdots, B\}$, let
$$g_{B}(i) = en\left(1 + \frac{e\cdot(n-B)}{B - i + 1}\right) + g_{<B}(i-1),$$
and for every $i \in \{1, \cdots, B - 1\}$, let
$$g_{<B}(i) = \frac{en}{B - i} + g_{B}(i).$$

For $t > 0$, let $X_t$ be the parent bit string of \ooea at iteration $t$. and let $T$ be the iteration number at which \ooea finds the optimum for the first time. Let 
\begin{equation}
f(X_t)=
    \begin{cases}
        g_{B}(LO(X_t)) & \text{if } |X_t|_1 = B,\\
        g_{<B}(LO(X_t) & \text{if } |X_t|_1 < B.
    \end{cases}
\end{equation} 


We consider two different cases, $|X_t|_1 = B$ and $|X_t|_1 < B$ and show in both the cases the drift is at least 1. Suppose we are in an iteration $t < T$ with $LO(X_t) = i$ and $|X_t|_1 = B$. Then the probability that the number of $1$s in the search point can decrease by $1$ in the next iteration is at least $\frac{B - i}{en}$. This is because we can get a desired search point by flipping only one of the $1$ bits of $B - i$, excluding the leading $1$s, and not flipping any other bit. Therefore,
\begin{align*}
E[f(X_{t+1}) - f(X_t) \bigm\vert & LO(X_t) = i, \text{ } |X_t|_1 = B]\\
    &\geq (g_{<B}(i) - g_{B}(i)) \cdot Pr(|X_{t+1}|_1 < B)\\
    &\geq \left(\frac{en}{B - i} + g_{B}(i) - g_{B}(i)\right) \cdot \left(\frac{B - i}{en} \right) = 1.
\end{align*}
Suppose we are in an iteration $t < T$ with $LO(X_t) = i$ and $|X_t|_1 < B$. Then in the next iteration the value of \LO can increase when the leftmost $0$ is flipped to $1$ as this does not violate the constraint. This happens with probability at least $\frac{1}{en}$. Since $|X_t|_1 < B$, we can also stay in the same level (same number of leading $1$s) and the number of $1$s can increase to $B$ with probability at most $\frac{n - B}{n}$ (see Lemma~\ref{lem:UBFB}). This implies that the potential can decrease by $\frac{en}{B - i}$ with probability at most $\frac{n - B}{n}$. 
\begin{align*}
\begin{split}
E[f(X_{t+1}) - &f(X_t) \bigm\vert LO(X_t) = i, |X_t|_1 < B]\\
    &\geq (g(i+1, B) - g_{<B}(i)) \cdot \frac{1}{en} - \left(\frac{en}{B - i} \cdot \frac{n - B}{n}\right)\\
    &\geq \left(en\left(1 + \frac{e\cdot(n-B)}{B - i}\right) + g_{<B}(i) - g_{<B}(i)\right) \cdot \left(\frac{1}{en} \right) \\
    &\qquad - \left(\frac{e\cdot(n - B)}{B - i} \right) \\
    &= 1.
\end{split}
\end{align*}
This results in an expected additive drift value greater than $1$ in all the cases, so according to the additive drift theorem \cite[Theorem~5]{ad},
\begin{align*}
\begin{split}
E[T] &\leq f(X_T) = g_B(B) \\
    &= \mathlarger{\mathlarger{\sum}}_{i=0}^{B-1} (g_{<B}(i) - g_{B}(i)) + \mathlarger{\mathlarger{\sum}}_{i=1}^{B} ( g_{B}(i) - g_{<B}(i - 1))\\
    &= \mathlarger{\mathlarger{\sum}}_{i=0}^{B-1} \frac{en}{B - i} +     \mathlarger{\mathlarger{\sum}}_{i=1}^{B} en\left(1 + \frac{e\cdot(n-B)}{B - i + 1}\right)\\
    &=  en\mathlarger{\mathlarger{\sum}}_{i=1}^{B} \left(\frac{1}{i}\right) + enB + e^2\cdot n(n-B) \mathlarger{\mathlarger{\sum}}_{i=1}^{B} \left(\frac{1}{i}\right)\\
    &\leq en(\log B + 1) + enB + e^2\cdot n(n-B)(\log B + 1)\\
    &= O(n^2 + n(n-B)\log B).
\end{split}
\end{align*}
\end{proof}


We now turn to the lower bound. When \ooea optimizes \LO in unconstrained setting the probability that a bit which is after the left-most $0$ is $1$ is exactly $\frac{1}{2}$. But this is not true in the constrained setting. The following lemma gives an upper bound on this probability during the cardinality constraint optimization.
\begin{lem}\label{lem:distributionInTail}
    For any $t \geq 0$, let $x^t$ denote the search point at iteration $t$ when \ooea is optimizing \LO with the cardinality constraint $B$. Then for any $t \geq 0$ and $i > LO(x^t)$, $Pr(x_i^t = 1 ) \leq 1/2$.
\end{lem}
\begin{proof}
    We will prove this by induction. The base case is true because we have an uniform random bit string at $t = 0$. Lets assume that the statement is true for $t$, i.e.  for any $i > LO(x_t)$, $Pr(x^t_i = 1 ) \leq 1/2$. Let $A$ be the event that the offspring is accepted. Then, for $i > LO(x_{t+1})$,
\begin{align*}
\begin{split}
Pr(x^{t+1}_i = 1 )
    &= Pr((x^t_i = 0) \cap (i^{th}\text{ bit is flipped}) \cap A) \\
    &+ Pr((x^t_i = 1) \cap (i^{th}\text{ bit is flipped}) \cap A^c) \\
    &+ Pr((x^t_i = 1) \cap (i^{th}\text{ bit is not flipped}).\\
\end{split}
\end{align*}    
Let $Pr(x^t_i = 1) = p$, $Pr( A \bigm\vert (i^{th}\text{ bit is flipped} \cap x^t_i = 0)) = a$ and $Pr(A \bigm\vert (i^{th}\text{ bit is flipped} \cap x^t_i = 1)) = b$. Then note that $a \leq b$ (because we have at least as many events as in probability $a$ contributing to the probability $b$) and by induction hypothesis,
\begin{align*}
\begin{split}
Pr(x^{t+1}_i = 1 )
    &= (1 - p)\cdot1/n \cdot a + p \cdot 1/n \cdot (1 - b) + p \cdot (1 - 1/n)\\
    &= a/n - (p\cdot a)/n + p/n - (p\cdot b)/n + p - p/n\\
    &= a/n - p \cdot (1 - a/n - b/n) \\
    &\leq a/n + 1/2 \cdot (1 - a/n - b/n) \\
    &\leq a/n + 1/2 \cdot (1 - a/n - a/n) = 1/2.
\end{split}
\end{align*}  
\end{proof}

We use the previous lemma to prove the $\Omega(n^2)$ lower bound on the expected time in the next theorem.

\begin{thm}\label{thm:lbn2}
Let $n, B \in \natnum$. Then the expected optimization time of the \ooea on the \LO with cardinality constraint $B$ is $\Omega\left(n^2\right).$ 
\end{thm}

\begin{proof}
We use the \textit{fitness level method with visit probabilities} technique defined in \cite[Theorem~8]{doerr2021lower} to prove this lower bound. Similar to \cite[Theorem~11]{doerr2021lower}, we also partition the search space $\bitstring$ based on the \LO values. For every $i \leq B$, let $A_i$ contain all the bit strings with the \LO value $i$. If our search point is in $A_i$, then we say that the search point is in the state $i$. For every $i \in \{1, \cdots, B-1\}$, we have to find the visit probabilities $v_i$ and an upper bound for $p_i$, the probability to leave the state $i$.

The best case scenario for the search point to leave the state $i$ is when the number of $1$s in the search point is less than $B$. In this case, we have to flip the $(i+1)^{th}$ bit to $1$ and should not flip any of the first $i$ bits to $0$. This happens with the probability $\frac{1}{n} \cdot \left(1 - \frac{1}{n}\right)^i$. Therefore, for every $i \in \{1, \cdots, B-1\}$, $p_i \leq \frac{1}{n} \cdot \left(1 - \frac{1}{n}\right)^i$.

Next, we claim that, for each $i \in \{1, \cdots, B-1\}$,  $v_i$ -- the probability to visit the state $i$ is at least $\frac{1}{2}$. We use \cite[Lemma~10]{doerr2021lower} to show this. Suppose the initial search point is in a state greater than or equal to $i$, then the probability for it to be in state $i$ is equal to the probability that the $(i+1)^{th}$ bit is $0$. Since the initial bit string is chosen uniformly at random the probability that the $(i+1)^{th}$ bit is $0$ is $\frac{1}{2}$. This shows the first required bound on the probability for the lemma in \cite[Lemma~10]{doerr2021lower}. Suppose  the search point is transitioning into a level greater than or equal to $i$, then the probability that it transition into state $i$ is equal to the probability that $(i+1)^{th}$ bit is $0$. From Lemma \ref{lem:distributionInTail}, we know that this probability is at least $1/2$. This gives the second bound required for the \cite[Lemma~10]{doerr2021lower}, therefore $v_i$ is at least $\frac{1}{2}$. 

By using fitness level method with visit probabilities theorem \cite[Theorem~8]{doerr2021lower}, if $T$ is the time taken by the \ooea to find an individual with $B$ number of \LO for the first time then, we have,
$ E[T] \geq \sum_{i = 0}^{B-1}\frac{v_i}{p_i} = \frac{n}{2}\cdot \mathlarger{\sum}_{i = 0}^{B-1}\left(1 - \frac{1}{n}\right)^{-i} \geq \frac{n^2}{2} = \Omega(n^2).$
\end{proof}
We aim to show the $\Omega(n^2 + n(n-B)\log B)$ lower bound and Theorem~\ref{thm:lbn2} gives the $\Omega(n^2)$ lower bound. Therefore, next we consider the case where $B$ is such that $n(n-B)\log B \neq O(n^2)$ 
to prove the desired lower bound.

\begin{thm}\label{thm:lb}
Let $n, B \in \natnum$ and suppose $n(n - B)\log B = \omega(n^2)$. Then the expected optimization time of the \ooea on the objective \LO with cardinality constraint $B$ is $\Omega\left(n(n-B)\log B\right).$ 
\end{thm}

\begin{proof}


We consider the potential function $g$ such that, for all $x \in \bitstring$,
$$
g(x) = \frac{n \cdot |x|_1 }{B - LO(x) + 1} + \sum_{i=LO(x)}^{B-1}\frac{n(n-B)}{32e^2(B-i)}.
$$
The first term appreciates progress by reducing the number of $1$s. This is scaled to later derive constant drift in expectation from such a reduction whenever $|x|_1 = B$, the case where progress by increasing the number of leading $1$s is not easy. The second term appreciates progress by increasing the number of leading $1$s, scaled to derive constant drift in case of $|x|_1 < B$.


The idea of the proof is as follows. We show that
the potential decreases by at most $10$ in expectation. Then the lower bound of additive drift theorem will give the desired lower bound on the expected run time (see \cite[Theorem~5]{ad}).

We start by calculating the expected potential at $t = 0$. Since the initial bit string is chosen uniformly at random the probability that the first bit is $0$ is $\frac{1}{2}$. Therefore $Pr(LO(x_0) = 0) = \frac{1}{2}$, which implies
 
 \begin{align*}
E[g(x_{0})] &\geq \frac{1}{2} \cdot E\left[ \sum_{i=LO(x_{0})}^{B-1}\frac{n(n-B)}{32e^2(B-i)} \bigm\vert LO(x_{0}) = 0\right] \\
&= \frac{n(n-B)}{64e^2} \sum_{i = 0}^{B-1}\frac{1}{B-i} \\
&= \frac{n(n-B)}{64e^2} \cdot \sum_{i= 1}^{B}\frac{1}{i} \geq \frac{n(n-B) \ln(B)}{64e^2}.
 \end{align*}
Therefore, there exits a constant $c > 0$ such that
$E[g(x_{0})] \geq cn(n-B) \log B.$ The optimum has a potential value of $nB$; thus, we can find a lower bound on the optimization time by considering the time to find a potential value of at most $nB$. Let $T = \min\{t \geq 0 \mid g(x_t) \leq nB\}$. Note that $T$ may not be the time at which we find the optimum for the first time. From $n(n - B)\log B = \omega(n^2)$ we get, for $n$ large enough, that $E[g(x_{0})] > nB$, which implies that the expected optimization time is at least $E[T]$.

In order to show the lower bound on the drift, we consider two different cases, $|x_t|_1 =  B$ and $|x_t|_1 < B$ and show in both the cases drift is at most $10$. First, we examine the case where the algorithm has currently $B$ number of $1$s. For any $t$, let $A_t$ be the event that $|x_t|_1 = B$ and 
let $\Delta_t = g(x_t) - g(x_{t+1})$ and 
\begin{align*}
\Delta^s_t & = \sum_{i=LO(x_t)}^{B-1}\frac{n(n-B)}{32e^2(B-i)} - \sum_{i=LO(x_{t+1})}^{B-1}\frac{n(n-B)}{32e^2(B-i)}\\
& = \sum_{i=LO(x_t)}^{LO(x_{t+1}) - 1}\frac{n(n-B)}{32e^2(B-i)}.
\end{align*}

\begin{align*}
\begin{split}
\text{Then, } E[\Delta_t \bigm\vert A_t] &= n\cdot E\left[\frac{|x_t|_1}{B - LO(x_t) + 1} - \frac{|x_{t+1}|_1}{B - LO(x_{t+1}) + 1} \bigm\vert A_t \right]\\
&\qquad + E[\Delta^s_t \bigm\vert A_t]\\
&\leq n\cdot E\left[\frac{|x_t|_1 - |x_{t+1}|_1}{B - LO(x_{t+1}) + 1} \bigm\vert A_t\right] + E[\Delta^s_t \bigm\vert A_t].
\end{split}
\end{align*}
Now we calculate the bounds for all the required expectations in the above equation. 





First we calculate a bound for $n\cdot E\left[\frac{|x_t|_1 - |x_{t+1}|_1}{B - LO(x_{t+1}) + 1} \bigm\vert A_t\right]$ by using the definition of the expectation. Let $I = \{0, \cdots, B - LO(x_t)\}$ and $J = \{-1, 0, \cdots, B - LO(x_{t+1}) - 1\}$. Then the possible values the random variable $|x_t|_1 - |x_{t+1}|_1$ can have are the values in $I$. And the possible values $B - LO(x_{t+1}) + 1$ can have are $\{B - LO(x_{t}) - j \mid j \in J\}$. For $i \in \{1, \cdots, B - LO(x_t)\}$, the probability $Pr((|x_t|_1 - |x_{t+1}|_1 = i) \cap (B - LO(x_{t+1}) + 1 = B - LO(x_{t}) + 1)) \leq \binom{B - LO(x_t)}{i} \left(\frac{1}{n}\right)^{i} \leq \frac{B - LO(x_t)}{i!n}$  and for $i \in \{1, \cdots, B - LO(x_t)\}$ and $j \in \{0, \cdots, B - LO(x_{t+1}) - 1\}\}$, the probability $Pr((|x_t|_1 - |x_{t+1}|_1 = i) \cap (B - LO(x_{t+1}) + 1 = B - LO(x_{t}) - j)) \leq \binom{B - LO(x_t)}{i} \left(\frac{1}{n}\right)^{i+1}\frac{1}{2^{j}} \leq \frac{B - LO(x_t)}{i!n^2} \frac{1}{2^{j}} $ (see Lemma \ref{lem:distributionInTail}). For $i \in I$ and $j \in J$, let $p_{t}^{ij} = Pr((|x_t|_1 - |x_{t+1}|_1 = i) \cap (B - LO(x_{t+1}) + 1 = B - LO(x_{t}) - j))$ and $K = J \setminus \{B - LO(x_{t}) + 1\}$ and $K^c = \{B - LO(x_{t}) + 1\}$. Then, $ n\cdot E\left[\frac{|x_t|_1 - |x_{t+1}|_1}{B - LO(x_{t+1}) + 1} \bigm\vert A_t\right] $

\begin{align*}
&= n \cdot \sum_{i \in I} \sum_{j \in J} \frac{i \cdot p_{t}^{ij}}{B - LO(x_{t}) - j} \\
&= n \cdot \sum_{i \in I} \sum_{j \in K^c} \frac{i \cdot p_{t}^{ij}}{B - LO(x_{t}) - j} + n \cdot \sum_{i \in I} \sum_{j \in K} \frac{i \cdot p_{t}^{ij}}{B - LO(x_{t}) - j} \\
&= \sum_{i \in I} \frac{i (B - LO(x_t))}{i! (B - LO(x_{t}) + 1)} + \sum_{i \in I} \frac{i (B - LO(x_t))}{i!n} \sum_{j \in K} \frac{\frac{1}{2^{j}} }{B - LO(x_{t}) - j} \\
& \leq \sum_{i \in I \setminus \{0\}} \frac{1}{(i - 1)!} + \sum_{i \in I}  \frac{1}{(i - 1)!} \sum_{j \in K} \frac{1}{2^{j}} \\
&\leq e + 2e = 3e \leq 9. \numberthis \label{eqn1}
\end{align*}

We used the infinite sum values $\sum_{i = 1}^{\infty} \frac{1}{(i - 1)!} = e$, $\sum_{i = 0}^{\infty} \frac{1}{2^i} = 2$, to bound our required finite sums in the above calculation.

Now, we calculate $E[\Delta^s_t \bigm\vert A_t]$, to get an upper bound for $E[\Delta_t \bigm\vert A_t]$. When $|x_t|_1 = B$, the probability to gain in the \LO-values  is at most $\frac{B- LO(x_t)}{n} \cdot \frac{1}{n}$. Therefore we have

\begin{align*}
E[\Delta^s_t \bigm\vert A_t] &= \frac{n(n-B)}{32e^2} \cdot E\left[\sum_{i=LO(x_{t})}^{LO(x_{t+1}) - 1}\frac{1}{B-i} \bigm\vert A_t \right]\\
&\leq \frac{n(n-B)}{32e^2} \cdot E\left[\frac{ LO(x_{t+1}) - LO(x_t)}{B - LO(x_{t+1}) + 1} \bigm\vert A_t \right].\\ \numberthis \label{expectation0}
\end{align*}
We calculate an upper bound for $E\left[\frac{ LO(x_{t+1}) - LO(x_t)}{B - LO(x_{t+1}) + 1} \bigm\vert A_t \right]$. The probability that $LO(x_{t+1}) - LO(x_t) = i$ given that we gain at least a leading one is the probability that next $i - 1$ bits after left-most $0$ bit) is $1$ followed by a $0$ bit. This implies that the probability that $LO(x_{t+1}) - LO(x_t) = i$ given that we gain at least a leading one is at most $\frac{1}{2^{i - 1}}$. Therefore, we have $E\left[\frac{ LO(x_{t+1}) - LO(x_t)}{B - LO(x_{t+1}) + 1} \bigm\vert A_t \right]$
\begin{align*}
 \leq \frac{B- LO(x_t)}{n} \cdot \frac{1}{n} \cdot \sum_{i = 1}^{B - LO(x_t) - 1} \frac{i \cdot 2^{1 - i}}{B - LO(x_{t}) - i}. \numberthis \label{expectation} 
\end{align*}
Equations \ref{expectation0} and \ref{expectation} imply that, $E[\Delta^s_t \bigm\vert A_t] $
\begin{align*}
&\leq \frac{n(n-B)}{32e^2} \cdot \frac{B- LO(x_t)}{n} \cdot \frac{1}{n} \cdot \sum_{i = 1}^{B - LO(x_t) - 1} \frac{i \cdot 2^{1 - i}}{B - LO(x_{t}) - i}\\
&\leq \frac{1}{16e^2} \cdot \sum_{i = 1}^{B - LO(x_t) - 1} \frac{(B - LO(x_t) )\cdot i}{(B - LO(x_{t}) - i) \cdot 2^{i}}\\
&= \frac{1}{16e^2} \cdot \sum_{i = 1}^{B - LO(x_t) - 1} \frac{(B - LO(x_t) - i + i)\cdot i}{(B - LO(x_{t}) - i) \cdot 2^{i}}\\
&\leq \frac{1}{16e^2} \left(\sum_{i = 1}^{B - LO(x_t) - 1} \frac{i}{2^{i}} + \sum_{i = 1}^{B - LO(x_t) - 1} \frac{i^2}{2^{i}}\right)\\
&\leq  \frac{1}{16e^2} (2 + 6) =  \frac{1}{2e^2} \leq 1. \numberthis \label{eqn2}
\end{align*}

We used the infinite sum values $\sum_{i = 1}^{\infty} \frac{i}{2^i} = 2$, $\sum_{i = 0}^{\infty} \frac{i^2}{2^i} = 6$, to bound our required finite sums in the above calculation.

 From Equations~\ref{eqn1} and~\ref{eqn2}, we have $E[\Delta_t \bigm\vert A_t] \leq 10$ which concludes the first case (when $|x_t|_1 = B$). Next we calculate the bound for the drift conditioned on the event $A_t^c$ (when $|x_t|_1 < B$).

\begin{align*}
\begin{split}
E[\Delta_t \bigm\vert A_t^c] &= n\cdot E\left[\frac{|x_t|_1}{B - LO(x_t) + 1} - \frac{|x_{t+1}|_1}{B - LO(x_{t+1}) + 1} \bigm\vert A_t^c \right]\\
&\qquad + E[\Delta^s_t \bigm\vert A_t^c]\\
&\leq n\cdot E\left[\frac{|x_t|_1 - |x_{t+1}|_1}{B - LO(x_{t+1}) + 1} \bigm\vert A_t^c\right] + E[\Delta^s_t \bigm\vert A_t^c].
\end{split}
\end{align*}
Similar to the previous case, for this case also we start by finding a bound for $n\cdot E\left[\frac{|x_t|_1 - |x_{t+1}|_1}{B - LO(x_{t+1}) + 1} \bigm\vert A_t^c\right]$. Let $\Delta_t^1 = |x_t|_1 - |x_{t+1}|_1$. Then $n\cdot E\left[\frac{\Delta_t^1}{B - LO(x_{t+1}) + 1} \bigm\vert A_t^c\right]$

\begin{align*}
& =  n\cdot E\left[\frac{\Delta_t^1}{B - LO(x_{t+1}) + 1} \bigm\vert A_t^c, \Delta_t^1> 0\right] \cdot Pr(\Delta_t^1 > 0) \\ 
\begin{split}
&\qquad + n\cdot E\left[\frac{\Delta_t^1}{B - LO(x_{t+1}) + 1} \bigm\vert A_t^c, \Delta_t^1 < 0\right] \cdot Pr(\Delta_t^1 < 0).\\
\end{split}
\end{align*}
Now we find upper bounds for both the quantities in the above equation. By doing calculations similar to the calculations which lead to the Equation (\ref{eqn1}), we get $n\cdot E\left[\frac{\Delta_t^1}{B - LO(x_{t+1}) + 1} \bigm\vert A_t^c, \Delta_t^1> 0\right] \leq 9$. Since there are at least $n - B$ number of $0$ bits, the probability to gain a $1$ bit is at least $\frac{n - B}{en}$. And the probability that $LO(x_t) = LO(x_{t+1})$ is at least $\frac{1}{2e}$, for $n$ large enough. Therefore, $n\cdot E\left[\frac{\Delta_t^1}{B - LO(x_{t+1}) + 1} \bigm\vert \Delta_t^1 < 0\right] \cdot Pr(\Delta_t^1 < 0) \leq -\frac{(n - B)}{2e^2(B - LO(x_{t}) + 1)}$. By combining these two bounds we have
\begin{align*}
n\cdot E\left[\frac{\Delta_t^1}{B - LO(x_{t+1}) + 1} \bigm\vert A_t^c\right] \leq 9 - \frac{(n - B)}{2e^2(B - LO(x_{t}) + 1)}.  \numberthis \label{eqn3}
\end{align*}

Next we calculate $E[\Delta^s_t \bigm\vert A_t^c]$, to get an upper bound for $E[\Delta_t \bigm\vert A_t^c]$. When $|x_t|_1 < B$, the probability to gain in \LO-value is at most $\frac{1}{n}$. Therefore,

\begin{align*}
E[\Delta^s_t & \bigm\vert A_t^c] \leq \frac{n(n-B)}{32e^2} \cdot E\left[\sum_{i=LO(x_{t})}^{LO(x_{t+1}) - 1}\frac{1}{B-i} \bigm\vert A_t^c \right]\\
&\leq \frac{n(n-B)}{32e^2} \cdot E\left[\frac{LO(x_{t+1}) - LO(x_t)}{B - LO(x_{t+1}) + 1} \bigm\vert A_t^c \right]\\
&\leq \frac{n(n-B)}{32e^2} \cdot \frac{1}{n} \cdot \sum_{i = 1}^{B - LO(x_t) - 1} \frac{i}{(B - LO(x_{t}) - i)\cdot 2^{i - 1}}\\
&= \frac{n-B}{32e^2(B - LO(x_t))}\cdot \sum_{i = 1}^{B - LO(x_t) - 1} \frac{(B - LO(x_t)) \cdot  i}{(B - LO(x_{t}) - i)\cdot 2^{i - 1}}\\
&= \frac{n - B}{16e^2(B - LO(x_t))} \cdot \sum_{i = 1}^{B - LO(x_t) - 1} \frac{(B - LO(x_t) - i + i)\cdot i}{(B - LO(x_{t}) - i) \cdot 2^{i}}\\
&\leq \frac{n - B}{16e^2(B - LO(x_t))} \left(\sum_{i = 1}^{B - LO(x_t) - 1} \frac{i}{2^{i}} + \sum_{i = 1}^{B - LO(x_t) - 1} \frac{i^2}{2^{i}}\right)\\
&\leq  \frac{n - B}{16e^2(B - LO(x_t))} (2 + 6) =  \frac{n - B}{2e^2(B - LO(x_t))}. \numberthis \label{eqn4}
\end{align*}
Since $B - LO(x_t) \geq 1$, we have $\frac{n - B}{2e^2(B - LO(x_t))} \leq \frac{n - B}{e^2(B - LO(x_t) + 1)}$. From Equations \ref{eqn3} and \ref{eqn4},we have

\begin{align*}
E[\Delta_t \bigm\vert A_t]  &\leq  9 - \frac{(n - B)}{2e^2(B - LO(x_{t}) + 1)} + \frac{n - B}{2e^2(B - LO(x_t))} \leq 10.
\end{align*}

Which concludes the second case (when $|x_t|_1 < B$). Now we have $E[\Delta_t \mid g(x_t)] \leq 10$. Therefore, by the lower bounding additive drift theorem \cite[Theorem~5]{ad},
$$E[T] \geq \frac{E[g(x_{0})] - nB}{10} = \Omega(n(n-B)\log B).$$
\end{proof}


\begin{cor}\label{cor:RuntimeOnConstrainedLO}
Let $n, B \in \natnum$. Then the expected optimization time of the \ooea on the \LO with cardinality constraint $B$ is $\Theta\left(n^2+n(n-B)\log B\right).$
\end{cor}
\begin{proof}
From Theorem \ref{thm:lbn2} and Theorem \ref{thm:lb} we have the required lower bound and we have the upper bound from Theorem \ref{thm:ub}. Therefore the expected optimization time is $\Theta\left(n^2+n(n-B)\log B\right).$ 
\end{proof}

\section{Better Run Times} \label{sec:sec4}

In this section we discuss two ways to obtain the (optimal) run time of $O(n^2)$.
First, we state a corollary to the proof of Theorem \ref{thm:ub}, that we can almost reach the bound within $O(n^2)$ iterations.
\begin{cor} \label{cor:cor1}
Let $n, B \in \natnum$ and $c > 0$. Then the \ooea on \LO with the cardinality constraint $B$ finds a search point with $B - c(n - B)$ leading $1$s within $O(n^2)$ in expectation.
\end{cor}

With the next theorem we show that incorporating the number of $0$s of a bit string as a secondary objective gives an  expected run time of the \ooea of $\Theta(n^2)$ to optimize cardinality constrained \LO.
\begin{thm}\label{thm:improved_algo}
 Let $B \leq n - 1$ and for any $x \in \bitstring$, let
 
 \[f(x) = \begin{cases} 
      (LO(x), |x|_0) & |x|_1 \leq B, \\
     - |x|_1 & \text{otherwise.}
   \end{cases}
\]
Then \ooea takes $\Theta(n^2)$ in expectation to optimize $f$ in the lexicographic order with the cardinality constraint $B$. 
\end{thm}

\begin{proof}
    For any $x \in \bitstring$, let
    $g(x) = 3eLO(x) + |x|_0,$
where $|x|_0$ represents the number of $0$s in $x$. Intuitively, we value both progress in decreasing the number of (unused) $1$s, as well as an increase in leading $1$s, but we value an increase in leading $1$s higher (since this is the ultimate goal, and typically comes at the cost of increasing the number of $1$ by a constant). Now we will show that $g(y) = 3eB + n - B $ if and only if $y$ is the optimum of $f$. Suppose for some $y \in \bitstring$, $g(y) = 3eB + n - B$. Then $3eLO(y) + |y|_0 =  3eB + n - B$, which implies that $3eLO(y) =  3eB + n - B - |y|_0$. Since $LO(y) \leq B$ and $|y|_0 \leq n - LO(y)$, $3eLO(y) =  3eB + n - B - |y|_0$ implies that $LO(y) =  B$. Therefore, $y$ is optimal. 

Let $T = \min \{t \geq 0 \bigm\vert g(x_t) \geq 3eB + n - B \}$. We will examine the drift at two different scenarios, $|x_t|_1 < B$ and $|x_t|_1 = B$ and show that in both the cases the drift is at least $1/n$. Let $\Delta_t = g(x_{t+1}) - g(x_t)$ and $A_t$ be the event that the left-most $0$ in $x_t$ is flipped. Then $E[\Delta_t \bigm\vert A_t^c] \geq 0$, because, if the number of \LO does not increase then $|x_{t+1}|_0 - |x_t|_0 \geq 0$ which in turn implies $\Delta_t \geq 0$. Therefore, for any $0 \leq  t < T$,
\begin{align*}
\begin{split}
E[\Delta_t \bigm\vert |&x_t|_1 < B] = E[\Delta_t \bigm\vert A_t, |x_t|_1 < B] \cdot Pr[A_t] \\ 
& \qquad + E[\Delta_t \bigm\vert A_t^c, |x_t|_1 < B] \cdot Pr[A_t^c]\\
& \geq \frac{1}{n} \cdot E[g(x_{t+1}) - g(x_t) \bigm\vert A_t, |x_t|_1 < B] + 0\\
&= \frac{1}{n} \cdot 3eE[LO(x_{t+1}) - LO(x_t) \bigm\vert A_t, |x_t|_1 < B] \\
&\qquad + \frac{1}{n} \cdot E[|x_{t+1}|_0 - |x_t|_0 \bigm\vert A_t, |x_t|_1 < B].\\
\end{split}
\end{align*}

Note that $E[LO(x_{t+1}) - LO(x_t) \bigm\vert A_t, |x_t|_1 < B]$ is greater than or equal to the probability of not flipping any other bits, since it increases the number of \LO by at least one. And $E[|x_t|_0 - |x_{t+1}|_0 \bigm\vert A_t, |x_t|_1 < B])$ is upper bounded by the sum $1 + \sum\limits_{i = 1}^{|x_t|_0 - 1} Pr(\text{flipping the $i^{th}$ $0$ bit})$. This is because we lose one $0$ bit by flipping the left-most $0$ bit and we flip each other 0-bit independently with probability $\frac{1}{n}$. And $\frac{|x_t|_0 - 1}{n} \leq 1$, therefore,

\begin{align*}
E[\Delta_t \bigm\vert |x_t|_1 < B] &\geq \frac{1}{n} \left(3e\left(1- \frac{1}{n}\right)^{n-1} - \left(1 + \frac{|x_t|_0 - 1}{n}\right)\right) \geq \frac{1}{n}. 
\end{align*}

This concludes the first case. Now, lets consider the case $|x_t|_1 = B$. Let $D$ be the event that the mutation operator flips exactly one $1$ bit which lies after the left-most $0$ bit and flips no other bits. Since $|x_t|_1 = B$ and $LO(x_t) < B$, there is at least one such $1$ bit, which implies $E[|x_{t+1}|_0 - |x_t|_0 \bigm\vert |x_t|_1 = B, D] \geq 1$. Also note that $Pr(D) \geq \frac{1}{en}$. If a search point is accepted, then the number of $1$ bits is at most $B$ and the \LO value cannot decrease; thus, $LO(x_{t+1}) \geq LO(x_{t})$ and $|x_{t+1}|_0 \geq n-B$. Overall we have $g(x_{t+1}) = 3eLO(x_{t+1}) + |x_{t+1}|_0 \geq 3eLO(x_{t}) + n-B = g(x_t)$. Therefore, $E[\Delta_t \bigm\vert |x_t|_1 = B, D^c] \geq 0$ and

\begin{align*}
\begin{split}
  E[\Delta_t \bigm\vert |x_t|_1 = B] & =  E[|x_{t+1}|_0 - |x_t|_0 \bigm\vert |x_t|_1 = B, D] \cdot Pr(D)\\
  & \qquad + E[\Delta_t \bigm\vert |x_t|_1 = B, D^c] \cdot Pr(D^c)\\
  & \geq \frac{1}{en}.
\end{split}
\end{align*}

The expected number of $0$s in the initially selected uniform random bit string is $\frac{n}{2}$ and the expected number of \LO is at least zero, therefore  $E[g(x_0)] \geq \frac{n}{2}$. We have an drift of at least $\frac{1}{en}$ in both the cases,  therefore we get the required upper bound by the additive drift theorem \cite[Theorem~5]{ad}, $$E[T] 
    \leq en \cdot (3eB + n - B - E[g(x_0)]) \leq 3e^2nB + \frac{en^2}{2} - enB = O(n^2).$$

This proves the upper bound. And the lower bound follows from Theorem~\ref{thm:lbn2}.
\end{proof}


\section{Empirical Analysis} \label{sec:experiments}

We want to extend our theoretical work on deterministic constraint the case of stochastic constraint models (as defined in Section~\ref{sc}). For the first model we use parameters $\mu = 1$ and $\sigma = 0.1$ and for the second model we use $\epsilon = \sqrt{3}$. Note that in the second model $U(B-\sqrt{3}, B+\sqrt{3})$ has variance $1$. For both the models we considered two different $B$ values 75 and 95 (also $B$ = 85 in the Appendix). As we will see, the \ooea struggles in these settings; in order to show that already a small parent population can remedy this, we also consider the $(10+1)$ EA in our experiments.

We use the following lemma for discussing certain probabilities in this section.
\begin{lem} \label{lem:erfc}
Let $k \in \{1, \cdots, B-1\}$, $x \in \{0,1\}^n, B \in [n]$, $W_x = \sum_{i = 1}^{n} x_i \cdot Y_i$ where $Y_i \sim N(1, \sigma^2)$ and $x_i$ be the $i-$th bit of $x$ and $|x|_1 \leq B - k$. Then $Pr(W_x > B) \leq \frac{1}{\sqrt{\pi}}e^{\frac{-k^2}{2n^2\sigma^2}}$ and $Pr(W_x > B \mid |x|_1 = B) =\frac{1}{2}$.
\end{lem}

In Figure \ref{fig:ooea_single} we have a single sample run of \ooea on the first model. We observe that if the \ooea finds a bit string with $B$ number of $1$s it violates the constraint with probability $\frac{1}{2}$ (see Lemma~\ref{lem:erfc}) and accepts a bit string with a lower  number of \LO. This process keeps repeating whenever the \ooea encounters an individual with a number of $1$s closer to $B$.
\begin{figure}[htbp]
  \centering
  \begin{minipage}[b]{0.4\textwidth}
    \includegraphics[width=\textwidth]{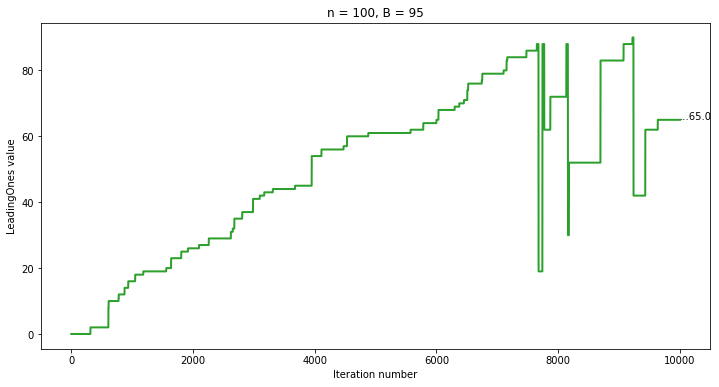}
    \caption{\ooea sample run with $n = 100$, $B = 85$ and $N(1, 0.1)$ chance constraint for $10000$ iteration.}
    \label{fig:ooea_single}
  \end{minipage}
\end{figure}

\begin{figure}[htbp]
  \centering
  \begin{minipage}[b]{0.4 \textwidth}
    \includegraphics[width=\textwidth]{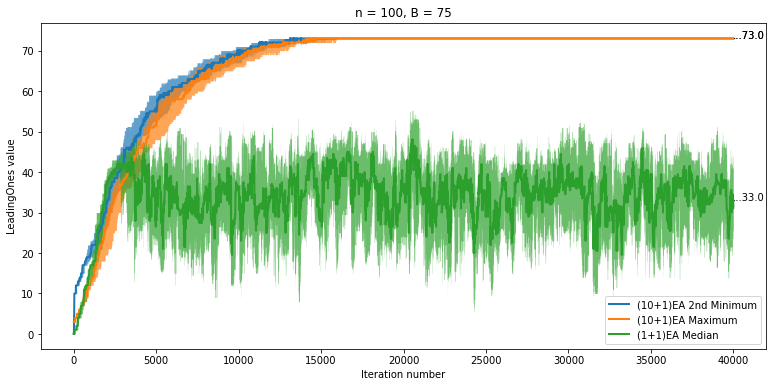}
    \caption{ (10+1) EA and \ooea on \LO with $n = 100$, $B = 75$ and $N(1, 0.1)$ chance constraint for $40000$ iterations.}
    \label{fig:B=75}
  \end{minipage}
\end{figure}

Figures \ref{fig:B=75} and \ref{fig:B=95} are about the first model in which we have the \LO-values of the best individual (bit string with the maximum fitness value) in each iteration of the (10+1) EA, the \LO values of the second-worst individuals (bit string with the second-smallest fitness value) in each iteration of the (10+1) EA and the \LO values at each iteration of the \ooea. Each curve is the median of thirty independent runs and the shaded area is the area between the $25-$th and the $75-$th quantile values. For all three $B$-values, after initial iterations, all the individuals except the worst individual in the (10+1) EA population have $B - 2$ number of leading $1$s. This is because, for this model, the probability that an individual with $B-2$ number of $1$s violates the constraint is at most $\frac{e^{-2}}{\sqrt{\pi}}$ (from Lemma~\ref{lem:erfc}).


\begin{figure}[htbp]
  \centering
  \begin{minipage}[b]{0.4\textwidth}
    \includegraphics[width=\textwidth]{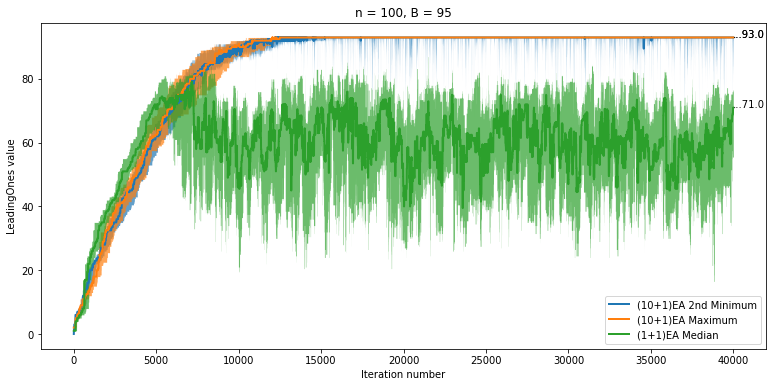}
    \caption{  (10+1) EA and \ooea on \LO with $n = 100$, $B = 95$ and $N(1, 0.1)$ chance constraint for $40000$ iterations.}
    \label{fig:B=95}
  \end{minipage}
\end{figure}

Figures \ref{fig:UB=75} and \ref{fig:UB=95} are about the second model and the curves represent the same things as in the previous figures but with respect to the second model. In these figures we can see that the best and the second worst individuals of the (10+1) EA are not the same because of the changing constraint values.

\begin{figure}[htbp]
  \centering
  \begin{minipage}[b]{0.4\textwidth}
    \includegraphics[width=\textwidth]{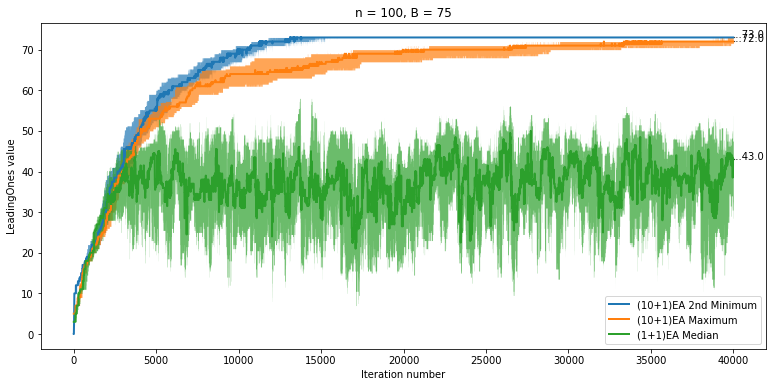}
    \caption{(10+1) EA and \ooea on \LO with $n = 100$, $B = 75$ and $U(B-\sqrt{3}, B+\sqrt{3})$ stochastic constraint for $40000$ iterations.}
    \label{fig:UB=75}
  \end{minipage}
\end{figure}


\begin{figure}[htbp]
  \centering
  \begin{minipage}[b]{0.4\textwidth}
    \includegraphics[width=\textwidth]{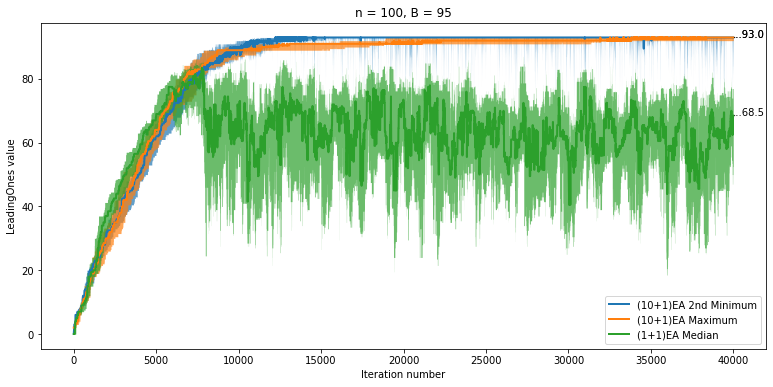}
    \caption{(10+1) EA and \ooea on \LO with $n = 100$, $B = 95$ and $U(B-\sqrt{3}, B+\sqrt{3})$ stochastic constraint for $40000$ iterations.}
    \label{fig:UB=95}
  \end{minipage}
\end{figure}

\section{Conclusions}
Understanding how evolutionary algorithms deal with constrained problems is an important topic of research. We investigated the classical LeadingOnes problem with additional constraints. For the case of a deterministic uniform constraint we have carried out a rigorous run time analysis of the (1+1)~EA which gives results on the expected optimization time in dependence of the chosen constraint bound. Afterwards, we examined stochastic constraints and the use of larger populations for dealing with uncertainties. Our results show a clear benefit of using the $(10+1)$~EA instead of the $(1+1)$~EA. We regard the run time analysis of population-based algorithms for our examined settings of stochastic constraints as an important topic for future work.

\section{Acknowledgements}
Frank Neumann has been supported by the Australian Research Council (ARC) through grant FT200100536. Tobias Friedrich and Timo Kötzing were supported by the German Research Foundation (DFG) through grant FR 2988/17-1.

\bibliographystyle{ACM-Reference-Format}
\bibliography{main.bib}

\appendix
\section{Appendix}
\begin{cor} \label{cor:cor1-appendix}
Let $n, B \in \natnum$ and $c > 0$. Then the \ooea on \LO with the cardinality constraint $B$ finds a search point with $B - c(n - B)$ leading $1$s within $O(n^2)$ in expectation.
\end{cor}

\begin{proof}
Let $x_t$ be the search point at the $t^{th}$ iteration of the \ooea  optimizing \LO with the cardinality constraint $B$ and $T_1 = \min\{t \geq 0 \bigm\vert LO(x_t) \geq B - c(n - B)\}$. And let $k = B - c(n - B)$. Then, from the proof of Theorem \ref{thm:ub} we know that
\begin{align*}
E[T_1] &\leq g_{<B}(k) \leq \frac{enB}{2} + (en + e^2 n(n-B))\left(1 + \ln(\frac{B}{c(n - B)})\right).\\
\end{align*}
Since $\ln{y} \leq y$ for any $y > 0$, we have $(n - B)\left(\ln(\frac{B}{c(n - B)})\right) \leq \frac{B}{c} = O(n)$. Therefore, $E[T_1] = O(n^2).$
\end{proof}

\begin{cor} \label{cor:cor1_appendix}
Let $n, B \in \natnum$ and $c > 0$. Then the \ooea on \LO with the cardinality constraint $B$ finds a search point with $B - c(n - B)$ number of \LO within $O(n^2)$ in expectation.
\end{cor}
\begin{proof}
Let $x_t$ be the search point at the $t^{th}$ iteration of the \ooea  optimizing \LO with the cardinality constraint $B$ and $T_1 = \min\{t \geq 0 \bigm\vert LO(x_t) = B - c(n - B)\}$. And let $k = B - c(n - B)$. Then from Theorem \ref{thm:ub} we know that,
\begin{align*}
E[T_1] &\leq g_{<B}(k) \\
    &= \mathlarger{\mathlarger{\sum}}_{i=0}^{k} (g_{<B}(i) - g_{B}(i)) + \mathlarger{\mathlarger{\sum}}_{i=1}^{k} ( g_{B}(i) - g_{<B}(i - 1))\\
    &= \mathlarger{\mathlarger{\sum}}_{i=0}^{B - c(n - B)} \frac{en}{B - i} + \mathlarger{\mathlarger{\sum}}_{i=1}^{B - c(n - B)} en\left(1 + \frac{e\cdot(n-B)}{B - i + 1}\right)\\
    &\leq  en\mathlarger{\mathlarger{\sum}}_{i=c(n - B)}^{B} \left(\frac{1}{i}\right) + \frac{enB}{2} + e^2\cdot n(n-B) \mathlarger{\mathlarger{\sum}}_{i=c(n - B)}^{B} \left(\frac{1}{i}\right)\\
    &\leq \frac{enB}{2} + (en + e^2 n(n-B))\left(1 + \ln(\frac{B}{c(n - B)})\right).\\
\end{align*}
Since $\ln{y} \leq y$ for any $y > 0$, we have $(n - B)\left(\ln(\frac{B}{c(n - B)})\right) \leq \frac{B}{c} \leq n$. Therefore, $E[T_1] = O(n^2).$
\end{proof}

\begin{lem} \label{lem:erfc_appendix}
Let $k \in \{1, \cdots, B-1\}$, $x \in \{0,1\}^n, B \in [n]$, $W_x = \sum_{i = 1}^{n} x_i \cdot Y_i$ where $Y_i \sim N(1, \sigma^2)$ and $x_i$ be the $i-$th bit of $x$ and $|x|_1 \leq B - k$. Then $Pr(W_x > B) \leq \frac{1}{\sqrt{\pi}}e^{\frac{-k^2}{2n^2\sigma^2}}$.
\end{lem}
\begin{proof}
First note that $W_x$ is nothing but sum of $|x|_1$ normal random variables with mean $1$ and variance $\sigma^2$, i.e. $W_x \sim N(|x|_1, |x|_1\cdot\sigma^2)$ and $Pr(W_x > B \mid |x|_1 = B) =\frac{1}{2}$.

\begin{align*}
Pr(W_x > B) 
    &= 1 - Pr(W_x \leq B)\\
    &= 1 - \frac{1}{2}\left(1+\erf\left(\frac{B - |x|_1}{|x|_1\cdot\sigma\sqrt{2}}\right)\right)\\
    &= \frac{1}{2}\left(1 - \erf\left(\frac{B - |x|_1}{|x|_1\cdot\sigma\sqrt{2}}\right)\right)\\
    &= \frac{1}{2}\erfc\left(\frac{B - |x|_1}{|x|_1 \cdot \sigma \sqrt{2}}\right).
\end{align*}

Since complementary error function $\erfc$ is a decreasing function and $|x|_1 \leq B - k$, we have 
\begin{align*}
Pr(W_x > B) \numberthis \label{eqn:1/2}
    &=  \frac{1}{2}\erfc\left(\frac{B - |x|_1}{|x|_1\cdot\sigma\sqrt{2}}\right)\\ 
    &\leq  \frac{1}{2}\erfc\left(\frac{k}{|x|_1\cdot\sigma\sqrt{2}}\right) \\ 
    &\leq \frac{1}{2}\erfc\left(\frac{k}{n\sigma\sqrt{2}}\right)
\end{align*}
Since $r_k = \frac{k}{n\sigma\sqrt{2}} > 0$, we can use the upper bound for the $\erfc$ from \cite{WAUB},
\begin{align*}
Pr(W_x > B)
    &\leq \frac{1}{2}\erfc\left(\frac{k}{n\sigma\sqrt{2}}\right)\\
    &\leq \frac{1}{\sqrt{\pi}}\frac{e^{-r_k^2}}{r_k +\sqrt{r_k^2 + \frac{4}{\pi}} }\\
    &\leq \frac{1}{\sqrt{\pi}}e^{\frac{-k^2}{2n^2\sigma^2}}.
\end{align*}

From Equation \ref{eqn:1/2}, we have $Pr(W_x > B \mid |x|_1 = B) = \frac{1}{2}\erfc\left(\frac{B - |x|_1}{|x|_1\cdot\sigma\sqrt{2}}\right)$ which is $\frac{1}{2}$.
\end{proof}

The following two figures are the experimental results for the constraint value $B = 85$.

\begin{figure}[htbp]
  \centering
  \begin{minipage}[b]{0.4\textwidth}
    \includegraphics[width=\textwidth]{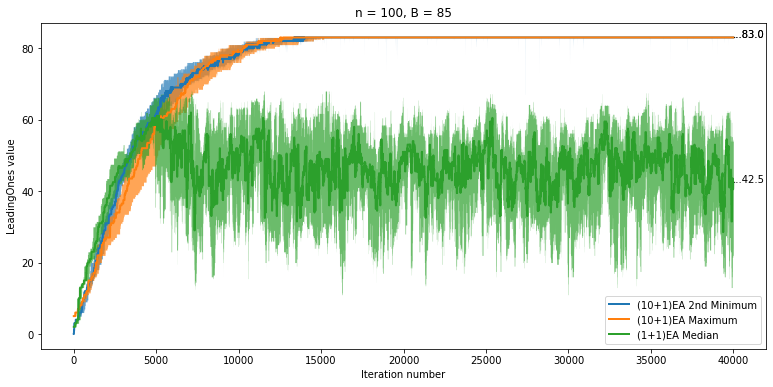}
    \caption{ (10+1) EA and \ooea on \LO with $n = 100$, $B = 85$ and $N(1, 0.1)$ chance constraint for $40000$ iterations.}
    \label{fig:B=85}
  \end{minipage}
\end{figure}

\begin{figure}[htbp]
  \centering
  \begin{minipage}[b]{0.4\textwidth}
    \includegraphics[width=\textwidth]{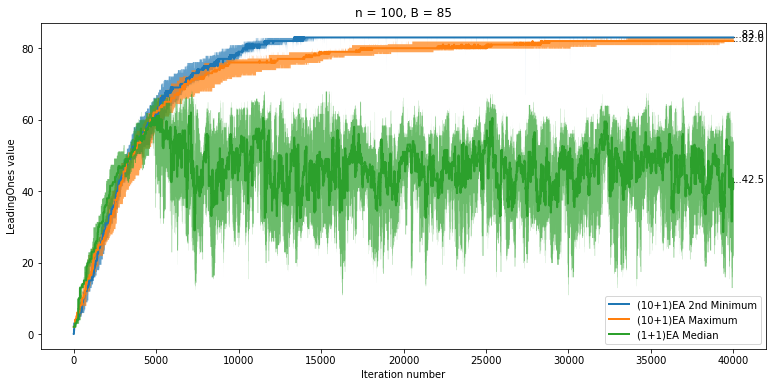}
    \caption{(10+1) EA and \ooea on \LO with $n = 100$, $B = 85$ and $U(B-\sqrt{3}, B+\sqrt{3})$ stochastic constraint for $40000$ iterations.}
    \label{fig:UB=85}
  \end{minipage}
\end{figure}

\end{document}

%% file: _used_packages.tex
\usepackage{float}
\usepackage{physics}
\usepackage{tabularx}
    \newcolumntype{Y}{>{\centering\arraybackslash}X}
\usepackage{booktabs}
\usepackage{pdflscape}
\usepackage{csvsimple}
\usepackage{hyperref}[hyperfootnotes=true]
\usepackage{url}
\usepackage{nameref}
\usepackage{tablefootnote}
\usepackage{color, colortbl}
\usepackage{mathtools}
\usepackage[algo2e, ruled, vlined, linesnumbered]{algorithm2e}
\usepackage{xparse}
\usepackage{makecell,booktabs}
\newcommand{\ignore}[1]{}

\usepackage
[
    noabbrev,   
    nameinlink, 
]
{cleveref} 

%% file: _macros.tex
\DeclareDocumentCommand{\set}{m g o}{
    \ensuremath{
        \IfNoValueTF{#3}{\left}{#3}\{#1
            \IfNoValueTF{#2}{}{
                \ \IfNoValueTF{#3}{\left}{#3}\vert\ \vphantom{#1}#2\IfNoValueTF{#3}{\right.}{}
            } \IfNoValueTF{#3}{\right}{#3}\}
    }\xspace
}

\definecolor{lightgreen}{rgb}{0.75,0.92,0.61}

\newcommand\numberthis{\addtocounter{equation}{1}\tag{\theequation}}

\newcommand{\labelname}[1]{
  \def\@currentlabelname{#1}}

\providecommand{\ignore}[1]{} 

\newcommand{\ooea}{(1+1) EA\xspace}

\newcommand{\muoea}{($\mu$+1) EA\xspace}

\DeclareMathOperator{\erfc}{erfc}

\newcommand*{\LO}{\mbox{\textsc{LeadingOnes}}\xspace}

\newcommand{\realnum}{\mathbb{R}}
\newcommand{\natnum}{\mathbb{N}}
\newcommand{\bitstring}{\{0,1\}^n}

\newtheorem{thm}{Theorem}
\newtheorem{lem}[thm]{Lemma}

\newtheorem{cor}[thm]{Corollary}

\newenvironment{myAlgorithm}%
	{
	\begin{center}
	\begin{algorithm2e}
	}%
	{
	\end{algorithm2e}
	\end{center}
	}
	
\allowdisplaybreaks

%% file: main.bbl

\begin{thebibliography}{29}


\ifx \showCODEN    \undefined \def \showCODEN     #1{\unskip}     \fi
\ifx \showDOI      \undefined \def \showDOI       #1{#1}\fi
\ifx \showISBNx    \undefined \def \showISBNx     #1{\unskip}     \fi
\ifx \showISBNxiii \undefined \def \showISBNxiii  #1{\unskip}     \fi
\ifx \showISSN     \undefined \def \showISSN      #1{\unskip}     \fi
\ifx \showLCCN     \undefined \def \showLCCN      #1{\unskip}     \fi
\ifx \shownote     \undefined \def \shownote      #1{#1}          \fi
\ifx \showarticletitle \undefined \def \showarticletitle #1{#1}   \fi
\ifx \showURL      \undefined \def \showURL       {\relax}        \fi
\providecommand\bibfield[2]{#2}
\providecommand\bibinfo[2]{#2}
\providecommand\natexlab[1]{#1}
\providecommand\showeprint[2][]{arXiv:#2}

\bibitem[\protect\citeauthoryear{Beyer and Sendhoff}{Beyer and
  Sendhoff}{2007}]%
        {beyer2007robust}
\bibfield{author}{\bibinfo{person}{Hans-Georg Beyer} {and}
  \bibinfo{person}{Bernhard Sendhoff}.} \bibinfo{year}{2007}\natexlab{}.
\newblock \showarticletitle{Robust optimization--a comprehensive survey}.
\newblock \bibinfo{journal}{\emph{Computer methods in applied mechanics and
  engineering}} \bibinfo{volume}{196}, \bibinfo{number}{33-34}
  (\bibinfo{year}{2007}), \bibinfo{pages}{3190--3218}.
\newblock


\bibitem[\protect\citeauthoryear{Charnes and Cooper}{Charnes and
  Cooper}{1959}]%
        {charnes1959chance}
\bibfield{author}{\bibinfo{person}{Abraham Charnes} {and}
  \bibinfo{person}{William~W Cooper}.} \bibinfo{year}{1959}\natexlab{}.
\newblock \showarticletitle{Chance-constrained programming}.
\newblock \bibinfo{journal}{\emph{Management science}} \bibinfo{volume}{6},
  \bibinfo{number}{1} (\bibinfo{year}{1959}), \bibinfo{pages}{73--79}.
\newblock


\bibitem[\protect\citeauthoryear{Doerr and K\"{o}tzing}{Doerr and
  K\"{o}tzing}{2021}]%
        {doerr2021lower}
\bibfield{author}{\bibinfo{person}{Benjamin Doerr} {and} \bibinfo{person}{Timo
  K\"{o}tzing}.} \bibinfo{year}{2021}\natexlab{}.
\newblock \showarticletitle{Lower Bounds from Fitness Levels Made Easy}. In
  \bibinfo{booktitle}{\emph{Proceedings of the Genetic and Evolutionary
  Computation Conference}} \emph{(\bibinfo{series}{GECCO 2021})}.
  \bibinfo{publisher}{Association for Computing Machinery},
  \bibinfo{address}{New York, NY, USA}, \bibinfo{pages}{1142–1150}.
\newblock
\showISBNx{9781450383509}
\urldef\tempurl%
\url{https://doi.org/10.1145/3449639.3459352}
\showDOI{\tempurl}


\bibitem[\protect\citeauthoryear{Doerr and Neumann}{Doerr and Neumann}{2020}]%
        {DBLP:series/ncs/2020DN}
\bibfield{editor}{\bibinfo{person}{Benjamin Doerr} {and} \bibinfo{person}{Frank
  Neumann}} (Eds.). \bibinfo{year}{2020}\natexlab{}.
\newblock \bibinfo{booktitle}{\emph{Theory of Evolutionary Computation - Recent
  Developments in Discrete Optimization}}.
\newblock \bibinfo{publisher}{Springer}.
\newblock
\showISBNx{978-3-030-29413-7}
\urldef\tempurl%
\url{https://doi.org/10.1007/978-3-030-29414-4}
\showDOI{\tempurl}


\bibitem[\protect\citeauthoryear{Droste, Jansen, and Wegener}{Droste
  et~al\mbox{.}}{2002}]%
        {DBLP:journals/tcs/DrosteJW02}
\bibfield{author}{\bibinfo{person}{Stefan Droste}, \bibinfo{person}{Thomas
  Jansen}, {and} \bibinfo{person}{Ingo Wegener}.}
  \bibinfo{year}{2002}\natexlab{}.
\newblock \showarticletitle{On the analysis of the {(1+1)} evolutionary
  algorithm}.
\newblock \bibinfo{journal}{\emph{Theor. Comput. Sci.}} \bibinfo{volume}{276},
  \bibinfo{number}{1-2} (\bibinfo{year}{2002}), \bibinfo{pages}{51--81}.
\newblock


\bibitem[\protect\citeauthoryear{Eiben and Smith}{Eiben and Smith}{2015}]%
        {DBLP:series/ncs/EibenS15}
\bibfield{author}{\bibinfo{person}{A.~E. Eiben} {and} \bibinfo{person}{James~E.
  Smith}.} \bibinfo{year}{2015}\natexlab{}.
\newblock \bibinfo{booktitle}{\emph{Introduction to Evolutionary Computing,
  Second Edition}}.
\newblock \bibinfo{publisher}{Springer}.
\newblock


\bibitem[\protect\citeauthoryear{Friedrich, K{\"{o}}tzing, Lagodzinski,
  Neumann, and Schirneck}{Friedrich et~al\mbox{.}}{2020}]%
        {DBLP:journals/tcs/FriedrichKLNS20}
\bibfield{author}{\bibinfo{person}{Tobias Friedrich}, \bibinfo{person}{Timo
  K{\"{o}}tzing}, \bibinfo{person}{J.~A.~Gregor Lagodzinski},
  \bibinfo{person}{Frank Neumann}, {and} \bibinfo{person}{Martin Schirneck}.}
  \bibinfo{year}{2020}\natexlab{}.
\newblock \showarticletitle{Analysis of the (1+1)~{EA} on subclasses of linear
  functions under uniform and linear constraints}.
\newblock \bibinfo{journal}{\emph{Theor. Comput. Sci.}}  \bibinfo{volume}{832}
  (\bibinfo{year}{2020}), \bibinfo{pages}{3--19}.
\newblock


\bibitem[\protect\citeauthoryear{Friedrich, Kötzing, Lagodzinski, Neumann, and
  Schirneck}{Friedrich et~al\mbox{.}}{2017}]%
        {lo}
\bibfield{author}{\bibinfo{person}{Tobias Friedrich}, \bibinfo{person}{Timo
  Kötzing}, \bibinfo{person}{J.~A.~Gregor Lagodzinski}, \bibinfo{person}{Frank
  Neumann}, {and} \bibinfo{person}{Martin Schirneck}.}
  \bibinfo{year}{2017}\natexlab{}.
\newblock \showarticletitle{Analysis of the {(1+1)} EA on Subclasses of Linear
  Functions under Uniform and Linear Constraints}. In
  \bibinfo{booktitle}{\emph{Foundations of Genetic Algorithms (FOGA)}}.
  \bibinfo{publisher}{ACM Press}, \bibinfo{pages}{45--54}.
\newblock


\bibitem[\protect\citeauthoryear{Hanasusanto, Roitch, Kuhn, and
  Wiesemann}{Hanasusanto et~al\mbox{.}}{2015}]%
        {hanasusanto2015distributionally}
\bibfield{author}{\bibinfo{person}{Grani~A Hanasusanto},
  \bibinfo{person}{Vladimir Roitch}, \bibinfo{person}{Daniel Kuhn}, {and}
  \bibinfo{person}{Wolfram Wiesemann}.} \bibinfo{year}{2015}\natexlab{}.
\newblock \showarticletitle{A distributionally robust perspective on
  uncertainty quantification and chance constrained programming}.
\newblock \bibinfo{journal}{\emph{Mathematical Programming}}
  \bibinfo{volume}{151} (\bibinfo{year}{2015}), \bibinfo{pages}{35--62}.
\newblock


\bibitem[\protect\citeauthoryear{He and Yao}{He and Yao}{2004}]%
        {ad}
\bibfield{author}{\bibinfo{person}{Jun He} {and} \bibinfo{person}{Xin Yao}.}
  \bibinfo{year}{2004}\natexlab{}.
\newblock \showarticletitle{A study of drift analysis for estimating
  computation time of evolutionary algorithms}.
\newblock \bibinfo{journal}{\emph{Natural Computing}}  \bibinfo{volume}{3}
  (\bibinfo{year}{2004}), \bibinfo{pages}{21--35}.
\newblock


\bibitem[\protect\citeauthoryear{Jansen}{Jansen}{2013}]%
        {DBLP:series/ncs/Jansen13}
\bibfield{author}{\bibinfo{person}{Thomas Jansen}.}
  \bibinfo{year}{2013}\natexlab{}.
\newblock \bibinfo{booktitle}{\emph{Analyzing Evolutionary Algorithms - The
  Computer Science Perspective}}.
\newblock \bibinfo{publisher}{Springer}.
\newblock
\showISBNx{978-3-642-17338-7}
\urldef\tempurl%
\url{https://doi.org/10.1007/978-3-642-17339-4}
\showDOI{\tempurl}


\bibitem[\protect\citeauthoryear{Li, Arellano-Garcia, and Wozny}{Li
  et~al\mbox{.}}{2008}]%
        {li2008chance}
\bibfield{author}{\bibinfo{person}{Pu Li}, \bibinfo{person}{Harvey
  Arellano-Garcia}, {and} \bibinfo{person}{G{\"u}nter Wozny}.}
  \bibinfo{year}{2008}\natexlab{}.
\newblock \showarticletitle{Chance constrained programming approach to process
  optimization under uncertainty}.
\newblock \bibinfo{journal}{\emph{Computers \& chemical engineering}}
  \bibinfo{volume}{32}, \bibinfo{number}{1-2} (\bibinfo{year}{2008}),
  \bibinfo{pages}{25--45}.
\newblock


\bibitem[\protect\citeauthoryear{Miller and Wagner}{Miller and Wagner}{1965}]%
        {miller1965chance}
\bibfield{author}{\bibinfo{person}{Bruce~L Miller} {and}
  \bibinfo{person}{Harvey~M Wagner}.} \bibinfo{year}{1965}\natexlab{}.
\newblock \showarticletitle{Chance constrained programming with joint
  constraints}.
\newblock \bibinfo{journal}{\emph{Operations Research}} \bibinfo{volume}{13},
  \bibinfo{number}{6} (\bibinfo{year}{1965}), \bibinfo{pages}{930--945}.
\newblock


\bibitem[\protect\citeauthoryear{Nair and Miller-Hooks}{Nair and
  Miller-Hooks}{2011}]%
        {nair2011fleet}
\bibfield{author}{\bibinfo{person}{Rahul Nair} {and} \bibinfo{person}{Elise
  Miller-Hooks}.} \bibinfo{year}{2011}\natexlab{}.
\newblock \showarticletitle{Fleet management for vehicle sharing operations}.
\newblock \bibinfo{journal}{\emph{Transportation Science}}
  \bibinfo{volume}{45}, \bibinfo{number}{4} (\bibinfo{year}{2011}),
  \bibinfo{pages}{524--540}.
\newblock


\bibitem[\protect\citeauthoryear{Neumann and Neumann}{Neumann and
  Neumann}{2020}]%
        {DBLP:conf/ppsn/NeumannN20}
\bibfield{author}{\bibinfo{person}{Aneta Neumann} {and} \bibinfo{person}{Frank
  Neumann}.} \bibinfo{year}{2020}\natexlab{}.
\newblock \showarticletitle{Optimising Monotone Chance-Constrained Submodular
  Functions Using Evolutionary Multi-objective Algorithms}. In
  \bibinfo{booktitle}{\emph{{PPSN} {(1)}}} \emph{(\bibinfo{series}{Lecture
  Notes in Computer Science}, Vol.~\bibinfo{volume}{12269})}.
  \bibinfo{publisher}{Springer}, \bibinfo{pages}{404--417}.
\newblock


\bibitem[\protect\citeauthoryear{Neumann, Xie, and Neumann}{Neumann
  et~al\mbox{.}}{2022}]%
        {DBLP:conf/ppsn/NeumannXN22}
\bibfield{author}{\bibinfo{person}{Aneta Neumann}, \bibinfo{person}{Yue Xie},
  {and} \bibinfo{person}{Frank Neumann}.} \bibinfo{year}{2022}\natexlab{}.
\newblock \showarticletitle{Evolutionary Algorithms for Limiting the Effect of
  Uncertainty for the Knapsack Problem with Stochastic Profits}. In
  \bibinfo{booktitle}{\emph{Parallel Problem Solving from Nature - {PPSN}
  {XVII} - 17th International Conference, {PPSN} 2022, Proceedings, Part {I}}}
  \emph{(\bibinfo{series}{Lecture Notes in Computer Science},
  Vol.~\bibinfo{volume}{13398})}. \bibinfo{publisher}{Springer},
  \bibinfo{pages}{294--307}.
\newblock
\urldef\tempurl%
\url{https://doi.org/10.1007/978-3-031-14714-2\_21}
\showDOI{\tempurl}


\bibitem[\protect\citeauthoryear{Neumann, Pourhassan, and Witt}{Neumann
  et~al\mbox{.}}{2021}]%
        {DBLP:journals/algorithmica/NeumannPW21}
\bibfield{author}{\bibinfo{person}{Frank Neumann}, \bibinfo{person}{Mojgan
  Pourhassan}, {and} \bibinfo{person}{Carsten Witt}.}
  \bibinfo{year}{2021}\natexlab{}.
\newblock \showarticletitle{Improved Runtime Results for Simple Randomised
  Search Heuristics on Linear Functions with a Uniform Constraint}.
\newblock \bibinfo{journal}{\emph{Algorithmica}} \bibinfo{volume}{83},
  \bibinfo{number}{10} (\bibinfo{year}{2021}), \bibinfo{pages}{3209--3237}.
\newblock


\bibitem[\protect\citeauthoryear{Neumann and Sutton}{Neumann and
  Sutton}{2019}]%
        {DBLP:conf/foga/0001S19}
\bibfield{author}{\bibinfo{person}{Frank Neumann} {and}
  \bibinfo{person}{Andrew~M. Sutton}.} \bibinfo{year}{2019}\natexlab{}.
\newblock \showarticletitle{Runtime analysis of the {(1} + 1) evolutionary
  algorithm for the chance-constrained knapsack problem}. In
  \bibinfo{booktitle}{\emph{Proceedings of the 15th {ACM/SIGEVO} Conference on
  Foundations of Genetic Algorithms, {FOGA} 2019}}. \bibinfo{publisher}{{ACM}},
  \bibinfo{pages}{147--153}.
\newblock
\urldef\tempurl%
\url{https://doi.org/10.1145/3299904.3340315}
\showDOI{\tempurl}


\bibitem[\protect\citeauthoryear{Neumann and Witt}{Neumann and Witt}{2010}]%
        {DBLP:books/daglib/0025643}
\bibfield{author}{\bibinfo{person}{Frank Neumann} {and}
  \bibinfo{person}{Carsten Witt}.} \bibinfo{year}{2010}\natexlab{}.
\newblock \bibinfo{booktitle}{\emph{Bioinspired Computation in Combinatorial
  Optimization}}.
\newblock \bibinfo{publisher}{Springer}.
\newblock
\showISBNx{978-3-642-16543-6}
\urldef\tempurl%
\url{https://doi.org/10.1007/978-3-642-16544-3}
\showDOI{\tempurl}


\bibitem[\protect\citeauthoryear{Qian, Shi, Yu, and Tang}{Qian
  et~al\mbox{.}}{2017}]%
        {DBLP:conf/ijcai/QianSYT17}
\bibfield{author}{\bibinfo{person}{Chao Qian}, \bibinfo{person}{Jing{-}Cheng
  Shi}, \bibinfo{person}{Yang Yu}, {and} \bibinfo{person}{Ke Tang}.}
  \bibinfo{year}{2017}\natexlab{}.
\newblock \showarticletitle{On Subset Selection with General Cost Constraints}.
  In \bibinfo{booktitle}{\emph{Proceedings of the Twenty-Sixth International
  Joint Conference on Artificial Intelligence, {IJCAI} 2017}}.
  \bibinfo{publisher}{ijcai.org}, \bibinfo{pages}{2613--2619}.
\newblock
\urldef\tempurl%
\url{https://doi.org/10.24963/ijcai.2017/364}
\showDOI{\tempurl}


\bibitem[\protect\citeauthoryear{Qian, Yu, and Zhou}{Qian
  et~al\mbox{.}}{2015}]%
        {DBLP:conf/nips/QianYZ15}
\bibfield{author}{\bibinfo{person}{Chao Qian}, \bibinfo{person}{Yang Yu}, {and}
  \bibinfo{person}{Zhi{-}Hua Zhou}.} \bibinfo{year}{2015}\natexlab{}.
\newblock \showarticletitle{Subset Selection by Pareto Optimization}. In
  \bibinfo{booktitle}{\emph{Advances in Neural Information Processing Systems
  28: Annual Conference on Neural Information Processing Systems 2015}}.
  \bibinfo{pages}{1774--1782}.
\newblock


\bibitem[\protect\citeauthoryear{Roostapour, Neumann, and Neumann}{Roostapour
  et~al\mbox{.}}{2022a}]%
        {DBLP:journals/tcs/RoostapourNN22}
\bibfield{author}{\bibinfo{person}{Vahid Roostapour}, \bibinfo{person}{Aneta
  Neumann}, {and} \bibinfo{person}{Frank Neumann}.}
  \bibinfo{year}{2022}\natexlab{a}.
\newblock \showarticletitle{Single- and multi-objective evolutionary algorithms
  for the knapsack problem with dynamically changing constraints}.
\newblock \bibinfo{journal}{\emph{Theor. Comput. Sci.}}  \bibinfo{volume}{924}
  (\bibinfo{year}{2022}), \bibinfo{pages}{129--147}.
\newblock
\urldef\tempurl%
\url{https://doi.org/10.1016/j.tcs.2022.05.008}
\showDOI{\tempurl}


\bibitem[\protect\citeauthoryear{Roostapour, Neumann, Neumann, and
  Friedrich}{Roostapour et~al\mbox{.}}{2022b}]%
        {DBLP:journals/ai/RoostapourNNF22}
\bibfield{author}{\bibinfo{person}{Vahid Roostapour}, \bibinfo{person}{Aneta
  Neumann}, \bibinfo{person}{Frank Neumann}, {and} \bibinfo{person}{Tobias
  Friedrich}.} \bibinfo{year}{2022}\natexlab{b}.
\newblock \showarticletitle{Pareto optimization for subset selection with
  dynamic cost constraints}.
\newblock \bibinfo{journal}{\emph{Artif. Intell.}}  \bibinfo{volume}{302}
  (\bibinfo{year}{2022}), \bibinfo{pages}{103597}.
\newblock
\urldef\tempurl%
\url{https://doi.org/10.1016/j.artint.2021.103597}
\showDOI{\tempurl}


\bibitem[\protect\citeauthoryear{Shi, Yan, and Neumann}{Shi
  et~al\mbox{.}}{2022}]%
        {DBLP:conf/ppsn/ShiYN22}
\bibfield{author}{\bibinfo{person}{Feng Shi}, \bibinfo{person}{Xiankun Yan},
  {and} \bibinfo{person}{Frank Neumann}.} \bibinfo{year}{2022}\natexlab{}.
\newblock \showarticletitle{Runtime Analysis of Simple Evolutionary Algorithms
  for the Chance-Constrained Makespan Scheduling Problem}. In
  \bibinfo{booktitle}{\emph{{PPSN} {(2)}}} \emph{(\bibinfo{series}{Lecture
  Notes in Computer Science}, Vol.~\bibinfo{volume}{13399})}.
  \bibinfo{publisher}{Springer}, \bibinfo{pages}{526--541}.
\newblock


\bibitem[\protect\citeauthoryear{Weisstein}{Weisstein}{[n.\,d.]}]%
        {WAUB}
\bibfield{author}{\bibinfo{person}{Eric~W. Weisstein}.}
  \bibinfo{year}{[n.\,d.]}\natexlab{}.
\newblock \bibinfo{title}{Complementary Error Function}.
\newblock
  \bibinfo{howpublished}{\url{https://mathworld.wolfram.com/Erfc.html}}.
\newblock
\newblock
\shownote{Accessed: 2023-02-01}.


\bibitem[\protect\citeauthoryear{Xie, Harper, Assimi, Neumann, and Neumann}{Xie
  et~al\mbox{.}}{2019}]%
        {DBLP:conf/gecco/XieHAN019}
\bibfield{author}{\bibinfo{person}{Yue Xie}, \bibinfo{person}{Oscar Harper},
  \bibinfo{person}{Hirad Assimi}, \bibinfo{person}{Aneta Neumann}, {and}
  \bibinfo{person}{Frank Neumann}.} \bibinfo{year}{2019}\natexlab{}.
\newblock \showarticletitle{Evolutionary algorithms for the chance-constrained
  knapsack problem}. In \bibinfo{booktitle}{\emph{Proceedings of the Genetic
  and Evolutionary Computation Conference, ({GECCO} 2019)}}.
  \bibinfo{publisher}{{ACM}}, \bibinfo{pages}{338--346}.
\newblock
\urldef\tempurl%
\url{https://doi.org/10.1145/3321707.3321869}
\showDOI{\tempurl}


\bibitem[\protect\citeauthoryear{Xie, Neumann, and Neumann}{Xie
  et~al\mbox{.}}{2020}]%
        {DBLP:conf/gecco/XieN020}
\bibfield{author}{\bibinfo{person}{Yue Xie}, \bibinfo{person}{Aneta Neumann},
  {and} \bibinfo{person}{Frank Neumann}.} \bibinfo{year}{2020}\natexlab{}.
\newblock \showarticletitle{Specific single- and multi-objective evolutionary
  algorithms for the chance-constrained knapsack problem}. In
  \bibinfo{booktitle}{\emph{Proceedings of the Genetic and Evolutionary
  Computation Conference, ({GECCO} 2020)}}. \bibinfo{publisher}{{ACM}},
  \bibinfo{pages}{271--279}.
\newblock
\urldef\tempurl%
\url{https://doi.org/10.1145/3377930.3390162}
\showDOI{\tempurl}


\bibitem[\protect\citeauthoryear{Xie, Neumann, Neumann, and Sutton}{Xie
  et~al\mbox{.}}{2021}]%
        {DBLP:conf/gecco/XieN0S21}
\bibfield{author}{\bibinfo{person}{Yue Xie}, \bibinfo{person}{Aneta Neumann},
  \bibinfo{person}{Frank Neumann}, {and} \bibinfo{person}{Andrew~M. Sutton}.}
  \bibinfo{year}{2021}\natexlab{}.
\newblock \showarticletitle{Runtime analysis of {RLS} and the {(1+1)} {EA} for
  the chance-constrained knapsack problem with correlated uniform weights}. In
  \bibinfo{booktitle}{\emph{Proceedings of the Genetic and Evolutionary
  Computation Conference, ({GECCO} 2021)}}. \bibinfo{publisher}{{ACM}},
  \bibinfo{pages}{1187--1194}.
\newblock


\bibitem[\protect\citeauthoryear{Zhang and Li}{Zhang and Li}{2011}]%
        {zhang2011chance}
\bibfield{author}{\bibinfo{person}{Hui Zhang} {and} \bibinfo{person}{Pu Li}.}
  \bibinfo{year}{2011}\natexlab{}.
\newblock \showarticletitle{Chance constrained programming for optimal power
  flow under uncertainty}.
\newblock \bibinfo{journal}{\emph{IEEE Transactions on Power Systems}}
  \bibinfo{volume}{26}, \bibinfo{number}{4} (\bibinfo{year}{2011}),
  \bibinfo{pages}{2417--2424}.
\newblock


\end{thebibliography}
